\documentclass[10pt]{article}

\usepackage{amsmath}
\usepackage{amssymb}
\usepackage{float}
\usepackage{scalefnt}
\usepackage{ifthen}
\usepackage{url}
\usepackage{graphicx}
\usepackage{colortbl}
\usepackage{multirow}
\usepackage{array}

\newtheorem{theorem}{Theorem}

\newtheorem{definition}[theorem]{Definition}

\newenvironment{proof}[1][Proof]{\noindent\textbf{#1.} }{\ \rule{0.5em}{0.5em}}

\title{Grounding Occam's Razor in a Formal Theory of Simplicity}
\author{Ben Goertzel \\ SingularityNET Foundation \\ Hong Kong}

\begin{document}

\maketitle

\begin{abstract}
While the Occam's Razor heuristic -- when in doubt, choose the simplest hypothesis -- is intuitively appealing, it obviously requires a theory of simplicity in order to become fully meaningful.   Toward this end, a  formal theory of simplicity is introduced, in the context of a "combinational" computation model that views computation as comprising the iterated transformational and compositional activity of a population of agents upon each other.  Conventional measures of simplicity in terms of algorithmic information etc. are shown to be special cases of a broader understanding of the core "symmetry" properties constituting what is defined here as a Compositional Simplicity Measure (CoSM).
 
This theory of CoSMs is extended to a theory of CoSMOS (Combinational Simplicity Measure Operating Sets) which involve multiple simplicity measures utilized together.    Given a vector of simplicity measures, an entity is associated not with an individual simplicity value but with a "simplicity bundles" of Pareto-optimal simplicity-value vectors.  

CoSMs and CoSMOS are then used as a foundation for a theory of pattern and multipattern, and a theory of hierarchy and heterarchy in systems of patterns.  The formulation of pattern and simplicity in terms of combinational operations is shown to bear fruit in terms of explicating the different ways that pattern hierachies can emerge -- e.g. via associativity of the core combinational operations underlying simplicity, or via the presence of abstract combinatory-logic-style operations in the vocabulary of core combinational operations.  It is proposed that in human-like minds, perception and action pattern networks achieve  hierarchy largely via associativity, whereas cognitive pattern networks achieve  hierarchy largely via combinatory-logic-style abstraction.   A formalization of the cognitive-systems notion of a "coherent dual network" interweaving hierarchy and heterarchy in a consistent way is presented.

The high level end result of this investigation is to re-envision Occam's Razor as something like: {\it In order for a system of inter-combining elements to effectively understand the world, it should interpret itself and the world in the context of some array of simplicity  measures obeying certain basic criteria (defined in terms of the elements and their combinational properties).  Doing so enables it to build up coordinated hierarchical and heterarchical pattern structures that help it interpret the world in a subjectively meaningful and useful way.  }.   Or phrased a little differently: {\bf when in doubt, prefer hypotheses whose simplicity bundles are Pareto optimal}, partly because {\bf doing so both permits and benefits from the construction of coherent dual networks comprising coordinated and consistent multipattern hierarchies and heterarchies}.
\end{abstract}

\mbox{}
\mbox{}
\begin{center}
\noindent {\bf {\it Simplicity was so complicated \\ I couldn't understand it} \\ -- Mariela Ivanova }
\end{center}
\mbox{}
\section{Introduction  }  

Everybody loves Occam's Razor, the heuristic often phrased as "When in doubt, choose the simplest option," and elegantly expressed by Albert Einstein via his maxim that theories should be "As simple as possible, but no simpler."  This sort of advice sounds intuitively sensible, but without some precise understanding of what "simplicity" means, it's not particularly crisp guidance.

My own interest in Occam's Razor arises largely from my work in artificial intelligence.  A host of theorists have argued for Occam's central role in AI -- going back to Ray Solomonoff in the late 1960s, whose theory of "Solomonoff induction" involves, essentially, AIs that understand the world via choosing the hypothesis represented by the shortest computer program \cite{Solomonoff1964a}.  Marcus Hutter \cite{Hutter2005} has built a rigorous theory of general intelligence under infinite or near-infinite computing resources, founded on this idea; and Eric Baum has argued the merits of similar ideas from a broad conceptual perspective \cite{Baum}.   

Occam's Razor has also been considered foundational in the philosophy of science, by many different thinkers \cite{Gauch2003}.  There is a lot of power in the idea that complex hypotheses, like the Ptolemaic epicycles, have been systematically cast aside in favor of simpler, more compact hypotheses like the Copernican model.

However, all these applications of Occam's Razor either rely on very specialized formalizations of the "simplicity" concept (e.g. shortest program length), or neglect to define simplicity at all.

In my own prior work I have been guilty of a similar laxity in specifying the meaning of "simplicity."  My work on the theory of complex and cognitive systems has relied substantially on the formal theory of pattern  \cite{GoertzelHP}, which relies on an assumed quantitative measure of simplicity -- but I've never said much about this measure, except to give some specific examples of potential measures (such as some involving algorithmic information).

In this paper  \footnote{This paper was dramatically revised in August 2020.   The conceptual direction and overall structure didn't change, and the overall framing discussion was left largely untouched (though not untweaked), but the specific formalisms used and the precise theorems presented were completely overhauled, so that on a technical basis it became a significantly different paper.   E.g. the prior version used a more particular model of computation rather than a general multi-combinatory-operator model, and also did not consider the case of non-associative operations. } I take the bull by the horns and -- after some conceptual exploration of the interpretation of the concept of simplicity in the context of the Occam's Razor heuristic -- give a fairly simple formal theory of what "simplicity" means.   There is a close conceptual connection between the ideas given here and algorithmic information theory, which is briefly sketched, but the two theories  have differences beyond emphasis.

The approach taken is not to specify one particular "correct measure of simplicity," but rather to introduce a class of computational models referred to as "combinational models," which appears particularly natural in the context of biological and cognitive systems, and also possesses universal Turing capability (or potentially even hyper-Turing capability).  In the context of combinational models, I indicate a set of formal criteria, along with the proposal that any reasonable measure of simplicity must obey these.  For instance, it's easy to see that program length and runtime both fulfill the criteria to qualify as a simplicity measure.  

Rather than selecting one ideal simplicity measure, it is proposed that often the right course is to consider a vector of multiple simplicity measures -- leading to a notion of multisimplicity, where multisimplicity path relies heavily on the notion of Pareto optimality.    Given a vector of simplicity measures, an entity is associated not with an individual simplicity value but with a "simplicity bundles" of Pareto-optimal simplicity-value vectors.  

The notion of simplicity thus formalized is then used to give a slightly revised definition of the concept of "pattern" previously defined in \cite{GoertzelHP} ... along with a new concept of multipattern associated with multisimplicity measures.   It is shown that a partial order may be obtained via interpreting $a<b$ to mean, roughly: $b$ is a part of a pattern (or multipattern) in $a$.  Given a set of entities, some of which are patterns (or multipatterns) in each other, one may then obtain a natural hierarchical structure for these patterns, in which elements higher in the hierarchy are patterns in elements lower in the hierarchy.  

To prove the existence of such a pattern hierarchy, however, one must make some sort of assumption regarding the combinational operations underlying the calculation of simplicity.  We identify different ways that pattern hierachies can emerge -- via associativity (and a related property we call cost-associativity) of the core combinational operations underlying simplicity, or via the presence of abstract combinatory-logic-style operations in the vocabulary of core combinational operations.  It is hypothesized that in human-like minds, perception and action pattern networks achieve exact hierarchy via associativity, whereas cognitive pattern networks achieve approximate hierarchy via combinatory-logic-style abstraction.

The pattern hierarchy is then shown to naturally spawn a metric structure, which may be interpreted as a heterarchy to complement the hierarchy, in the spirit of the "dual network" concept discussed in \cite{GoertzelHP}.   A formal analysis of what makes coupled hierarchical and heterarchical structures into a "coherent dual network" is presented -- the first time this informal cognitive-systems concept has been given a rigorous treatment.

The formal axioms and manipulations presented here are relatively elementary, however, they represent an attempt to find an elegant formulation for some of the most foundational concepts underlying the study of complex and intelligent systems: simplicity, pattern, hierarchy and heterarchy.  

These ideas allow us to re-envision Occam's Razor as a philosophical principle something like the following: {\it In order for a system of inter-combining elements to effectively understand the world, it should interpret itself and the world in the context of some array of simplicity  measures obeying certain basic criteria.  Doing so enables it to build up coordinated hierarchical and heterarchical pattern structures that help it interpret the world in a subjectively meaningful and useful way.}   

Or, put a little differently, as an extension of the traditional "when in doubt, prefer the simpler hypothesis" version of Occam's Razor, we arrive at a more sophisticated and flexible maxim {\bf when in doubt, prefer hypotheses whose simplicity bundles are Pareto optimal} , and an observation that {\bf doing so both permits and benefits from the construction of coherent dual networks comprising coordinated and consistent multipattern hierarchies and heterarchies}.

\section{The Interpretation of Simplicity in the Context of Occam's Razor: General Background Discussion}

"Occam's Razor" -- the principle that, when in doubt, one should choose the simpler hypothesis -- is a powerful heuristic, useful in many contexts.   However, in order to escape vacuity, Occam's Razor requires some sort of understanding of {\it what simplicity is}.   One option is to take simplicity as a core, unanalyzed primitive -- but in that case, it is still necessary to explore and understand the properties of this primitive notion.   Another option is to ground the notion of simplicity in some other set of concepts. 

This section, a prelude to the formal investigations of the following sections, looks at some of the approaches that have been taken to formally grounding the notion of simplicity in ideas from other scientific and mathematical domains.   The three most fully fleshed out options are: 

\begin{itemize}
\item	Grounding via physics, so that the Razor basically says "choose the hypothesis that can be created/computed using the least energy" (or some other physical quantity)
\item	Grounding via cognitive theory, i.e. via stating that an intelligent system should try to model the world in terms of some sort of suitable measure of simplicity (I call this the Meta-Razor below)
\item Grounding via computing theory, i.e. giving a formal computational theory of simplicity and then asserting foundational status for the view of the universe as computational
\end{itemize}

\noindent  These are not the only possible approaches, but they are the ones one sees most often in the scientific literature.  

Hutter \cite{Hutter2005} and Schmidhuber \cite{FrontierSearch} appear basically sympathetic to the computationalist perspective, whereas my own tendency is to view the cognitive perspective as more fundamental, and to seek interpretations of the other two perspectives in a cognitive light.  In any case, the interrelationships between the three perspectives are fascinating and well worth exploring.   I will focus here mainly on the latter two approaches, exploring their connections carefully, and also noting some possible connections between them and the physics view.

\subsection{The Difficulty of Defining Simplicity}

The difficulty of defining simplicity an an "objective" way has been observed frequently by philosophers of science.  The  notion that scientists tend to prefer simpler theories is complicated by the fact that different scientific paradigms will tend to assess simplicity differently.   Historically, there has been no standard, agreed-upon formalization of what makes one theory simpler than another.  For instance, there is no uniform language in terms of which all scientists will agree to express their theories, agreeing to judge simplicity via length of expression in that language.  But merely saying that "scientists will tend to choose the theories that they personally, psychologically find simpler" is a conclusion of rather limited power -- because there are obviously factors other than simplicity that drive scientists' individual choices, and because scientists are diverse human beings with diverse psychologies, and may have all manner of irrational or idiosyncratic dynamics guiding their judgments of simplicity.

Mathematical formulations of Occam's Razor, as are presented in the formal theory of AI (Solomonoff, Hutter, Schmidhuber, etc.), do not show us any way out of this dilemma -- though they do present the dilemma in a clearer and more obvious form.  

Suppose one agrees to restrict attention to hypotheses that are expressible as computer programs -- so that Occam's Razor comes down to choosing the simplest computer program that is capable of producing some given set of data.    In this case, still one has the question -- what does "simpler" really mean?

For starters, one is confronted with a bewildering variety of possible measures, including program length \cite{Chaitin2008}, program runtime \cite{Bennett1990}, various weighted combinations of length and runtime \cite{FrontierSearch}, "sophistication"  \cite{Koppel1987}, and so forth.  And worse of all, once one has chosen one of these measures, one still has to choose a "reference Universal Turing Machine."   That is, if one is measuring simplicity as program length (a la classic Solomonoff induction), one still has to answer the question "Program length in what programming language?", or equivalently "Program length in machine language on what computer?"   

Theoretical computer science tells us that, in principle, the choice of reference machine doesn't matter -- for sufficiently complex entities, the simplicity will come out roughly the same no matter what reference machine is chosen.   However, the definition of "sufficiently complex" depends on the reference machine in question: the theorem is that, for any pair of reference machines $X$ and $Y$, there is some constant $C$ so that simplicity measured via $X$ and simplicity measured via $Y$ will be identical within amount $C$.   In practical terms this kind of equivalence is not necessarily all that helpful, and it can make a big difference what reference machine is chosen.  Because real-world intelligence is largely about computational efficiency -- about making choices in real, bounded situations using bounded space and time resources.   

In the context of real-world AI, for Occam's Razor based on program space to be meaningful, one needs to restrict attention to programs within some specified region of program space (defined so that, for any two programs in the region, the simulation constant C is not that large).  But then, the question is faced: how does one determine this region of program space?   If one does so via Occam's Razor -- "choose the simplest reference machine", i.e. "choose the simplest measure of simplicity"! -- then one just arrives at a regress.   If one does so by some other principle, then this fact certainly needs to be highlighted, as it's at least equally important to Occam's Razor in the process of hypothesis choice.

\subsection{Grounding Occam's Razor in Physics}

One way to sidestep this Razor regress is to ground simplicity in the physical universe -- arguing that in many contexts the appropriate reference machine is, in essence, the laws of physics as currently understood, which happen to assign a fundamental role to energy minimization.  For instance, one could say that the simplest program is the one that can be represented via the physical system utilizing the least amount of energy.  Or, taking a cue from Bennett's theory of logical depth, one could look at the system whose construction requires the least amount of energy.  Or some sort of combination of the energy required for construction and the energy required for operation.

The recourse to physics is perfectly reasonable, in some contexts.  However, it doesn't quite resolve the matter, in spite of pointing in some interesting directions.

Firstly, this avenue doesn't leave any room for positing Occam's Razor as a fundamental principle for understanding the universe.   If one argues that simplicity bottoms out in physics, then does one also want to argue that physics is discovered by choosing the simplest theory?   Obviously this is just another Razor regress.  

If one is willing to assume a certain theory of physics as true and correct, at least provisionally, then relative to this physics theory, one can formulate objective versions of Occam's Razor.  This approach may indeed be useful. Read Montague \cite{Montague2006} and others have argued that aspects of brain function can be understood via thinking of the brain as a solution evolution has found for the problem of achieving effective organism control using minimum energy.   In practical computing, minimizing space and time resource utilization of algorithms can be viewed as bottoming out in minimizing the energy required to build and operate the physical computers running the algorithms.

However, I believe additional insight may be obtained by pursuing an alternate perspective, in which intelligence rather than physics is primary.

\subsection{Grounding Occam's Razor in Intelligence}

Suppose one takes the view of an intelligent system modeling a stream of data, with constructing a "physics" of the world underlying that data considered as one of the tasks the system confronts.  In that general setting, not all "physics" theories the system might construct, would necessarily give rise to Occam's Razor type principles.   Our current standard theories of physics, involve a notion of energy minimization, but not all possible physics theories, broadly speaking, need involve minimization of some quantity in an equally straightforward way.

One might then postulate a principle that: A good theory of the world is one that describes real-world systems as minimizing some quantity, so that this quantity can be taken as the "simplicity" measure grounding Occam's Razor.   

That is, taking the perspective of an intelligent system choosing between different hypotheses about the world around it, one principle this system might follow is: Choose an hypothesis that leads to a world-model possessing an intuitively comprehensible Occam's Razor principle.

But, on what grounds would said intelligent system consider such a choice "good"?  Suppose the system has certain goals that are important to it.  Then, presumably, it would have to judge that choosing an Occam-friendly world-model is a good way for it to achieve its goals.

Obviously, not all systems, environments and goals will lead to this conclusion.  Some system/environment/goal triples will be more "Occam-friendly" than others.

One could formalize this notion in the context of reinforcement learning and statistical decision theory (among other approaches).  Consider an agent receiving percepts and rewards from, and enacting actions in, some environment.   Will it benefit the agent to model its environment in a way that causes it to choose actions according to some measure of "simplicity"?   Marcus Hutter's results regarding AIXI \cite{Hutter2005} give a positive answer -- in the case that the agent has very large computing resources, and a very large amount of time to interact with the world, and the sequence of rewards is determined by some computable function.  But given realistic computing resources and interaction time, it's not so clear.  Hutter has shown that the large-resources, long-time scenario is Occam-friendly under some reasonable assumptions -- but still, that's not the scenario in which we live.

However, we humans do constitute another piece of data -- we have limited resources, and have had limited time to interact with the world, and we have also come to model the world in an Occam-friendly manner, via our current theories of physics. 

The idea that an intelligent system should try to model the world in such a way that, for some comprehensible simplicity measure, Occam's Razor applies, might be labeled the "Meta-Razor."  Unlike Occam's Razor, which relies on some additional assumption about simplicity, the Meta-Razor relies only on the broader, weaker assumption that {\it measures of simplicity in general} are useful tools for modeling the world.  But the Meta-Razor lacks the power of Occam's Razor, because it leaves the door open for different intelligent systems to construct world-models involving different simplicity measures, and thus to rate hypotheses in different, incommensurate ways.

\subsection{Relations Between Physical and Computational Measures of Simplicity}

There are interesting relationships between the various approaches to understanding simplicity.  In the following, we will give a formal theory of simplicity motivated largely by the notion of the Meta-Razor, and then explore its relationship to computation theory, showing a close connection.  And, though we will not dwell on them extensively here, there are also connections between the computational and physical notions of complexity, that are worthy of future exploration.

For instance, there are mathematical relationships between algorithmic information (the length of the shortest program for computing something) and Shannon entropy, which indicate that the algorithmic information view of Occam's Razor and the Maximum Entropy prior (resulting from thermodynamics) \cite{Jaynes2003} are equivalent in some cases; e.g. \cite{Zurek1989} argues that "Algorithmic randomness is typically very difficult to calculate
exactly but relatively easy to estimate. In large systems,
probabilistic ensemble definitions of entropy (e.g., coarse-grained
entropy of Gibbs and Boltzmann's entropy $H=ln(W)$, as well as Shannon's
information-theoretic entropy) provide accurate estimates of the
algorithmic entropy of an individual system or its average value for
an ensemble."

Zurek \cite{Zurek1998} goes further, and argues that "Physical entropy É is a sum of (i) the missing information
measured by Shannon's formula and (ii) of the algorithmic information content -- algorithmic randomness -- present in the available data about the system. "   It is interesting to ask, for what kinds of practical inferences by a physically embodied cognitive system does physical entropy as interpreted by Zurek constitute a valid simplicity measure.   This is a speculative direction that we will not explore fully here, but is presented as a possible guide for future investigation.

\subsection{The Value of Grounding Simplicity}

Occam's Razor is a powerful heuristic -- but unless one wishes to take "simplicity" as a basic ontological primitive, it doesn't stand on its own.  It requires some external grounding, because the notion of "simplicity" must come from somewhere.    One can ground simplicity "objectively", via assuming a physical world involving a simplicity measure such as energy, or via assuming a universe governed by a particular computational model; or one can ground it "subjectively" via the Meta-Razor, as a choice that an intelligent agent makes, to model the world in such a way that a comprehensible simplicity measure emerges from that world.   Either way the specific grounding of simplicity is a highly significant matter, without which the Razor is vacuous.  

The remainder of this paper pursues a formalization of the concept of simplicity, inspired largely by the subjectivist approach to specifying Occam's Razor.  A set of criteria is presented, specifying what constitutes a sensible measure of simplicity.  This fits naturally with the hypothesis that, to understand the world and itself, a mind should construct a perspective centered around some sensible measure of simplicity, fulfilling the criteria.  This notion of simplicity is then connected with computation theory, showing that computation-theoretic measures of simplicity are in essence a special case of the broader axiomatic notion presented.   The remaining  sections of the paper then explore how a notion of simplicity may be used by an intelligent system to build up a fuller mathematical framework for understanding the world -- e.g. a hierarchy and a metric space of patterns recognized in experience.

\section{A General Formalization of Simplicity in the Context of Combinational Systems}

\subsection{A Combinational Foundation for Computational, Biological and Cognitive Systems}

To formulate a concrete model of simplicity, one must make some basic assumptions about the nature of the systems or models whose simplicity is being measured.   Here I will assume what might be called a {\it combinational model} of computation -- i.e. I will look at systems that are composed of a set of elements that act on and transform each other to produce other elements, and join with each other to produce new elements.    This is conceptually about the same as what I have called a "self-generating system" in prior publications \cite{Goertzel1994} \cite{GoertzelHP} \cite{EGI1}, but some of the formal details will be worked differently here.  This sort of computational model also connects with Weaver's work on Open-Ended Intelligence \cite{weinbaum2017open} and I believe can serve as a way of thinking about intelligence that spans the classical computation-theory framework and the more self-organizing-systems oriented  open-ended intelligence approach.

Consider a space $\mathcal{E}$ of entities endowed with a set of binary operations $*_i : \mathcal{E} \rightarrow \mathcal{E}^{k_i}, i = 1 \ldots K$.   The operations $*_i$ may be thought of e.g. as reactions via which pairs of entities react to produce sets of entities, or as combinatory operators via which pairs of entities combine to produce sets of new entities.   Let $\mathcal{A}$ denote the set of entities that are "atomic" in the sense that they cannot be produced via binary combinations of any other entities.   We also make the formal assumption that the space $\mathcal{E}$ has an element $e$ that works as an identity for all the $*_i$.  \footnote{ It would also be perfectly reasonable to introduce ternary or higher order operators (so as to more straightforwardly. model e.g. catalysis phenomena) but these can be reduced to binary operators, and so we have decided to stick with binary for the current exposition.}

The case $k_i > 1$ corresponds to the situation where e.g. a chemical reaction produces more than one type of chemical as a product.   One may also consider filtration operations $*_f$ so that e.g.  

  \begin{equation}
   x *_f y =
    \begin{cases}
      y, & \text{if}\ x=y \\
      e, & \text{otherwise}
    \end{cases}
  \end{equation}
  
  \noindent Considering the cost of filtration operations prevents untoward use of cheap operators with large $k_i$ that just produce everything under the sun -- in this case the cheapness of production may be compensated by the cost of filtering to get just the results one wants.

There is a literature using this kind of multi-operator model to give a more biology-oriented infrastructure for computing, as an alternative to Turing machine or formal logic type models.  Cardelli and Zavattaro \cite{cardelli2010turing} and Danos and Laneve \cite{danos2003core} present two variations in this direction.   In each of these approaches, a reaction operation $*_1$ and a complexation/polymerization operation $*_2$ are introduced, with specific constraints inspired by the nature of chemical bonds and compounds.   The Cardelli and Zavattaro model also introduces a third operator that deals with splitting of polymers into parts.   Both papers present models that are provably Turing complete, and make the interesting point that complexation (the building of larger structures out of smaller ones) is necessary to achieve Turing completeness in this sort of framework.   If one just has reaction networks, one can get an arbitrarily close approximation to any Turing computable function with a sufficiently large network, but one can't exactly emulate every Turing computable function.   Adding the ability to construct large compounds of atomic entities, one can build extensible vaguely DNA-type structures which can then serve the rough role of Turing machine tapes.

For our considerations here, we don't need to commit to any particular computational model of this nature, however it is important to know that there are indeed reasonably simple and elegant models of this nature that have universal Turing capability.   The details of the Cardelli/Zavattaro and Danos/Laneve models are heavily biochemistry and molecular biology oriented, but this aspect of their work is largely orthogonal to our concerns here.  One could straightforwardly formulate variations of their constructs and arguments that are more cognitive than chemical in specifics; this would tie in more closely with the present considerations, but wouldn't change the concrete arguments we'll make below.

It is also interesting to consider each specific instance of an operator $*_i$ as a sort of autonomous agent $\hat{*}_i$, capable of locating its arguments and making its result identifiable.   To achieve this we may assume that each element of  $x \in \mathcal{E}$  is associated with a certain address $ \alpha(x)$ drawn from an address space $\mathcal{A}$ (which could for instance be the space of finite bit strings).   An autonomous operator-instance or "auto-op" $\hat{*}_i$ is then a tuple of the form $(*_i, \alpha_1, \alpha_2)$, where $\alpha_1, \alpha_2$ are the addresses of the operator's two arguments.   We can assume that the output of the application of the operator-instance to its arguments is assigned an address that is uniquely determined by the argument addresses and the operator index $i$.   Let $\mathcal{O}$ denote the space of auto-ops derived from the various $*_i$ over $\mathcal{E}$.

\subsection{Compositional Simplicity Measures} 

Next, suppose one has quantitative measures $\sigma: \mathcal{E} \rightarrow [0,\infty)$ and $\sigma^*: \mathcal{O} \rightarrow [0,\infty)$ (understood intuitively as measuring the simplicity of entities and auto-ops respectively).   We will say that the pair $(\sigma, \sigma^*)$ is a Compositional Simplicity Measure or CoSM if 

$$
\sigma(x) =  min_{y,z,i: x = y *_i z }  h(y,z)
$$

\noindent where

$$
h(y,z) = \sigma(y) + \sigma(z) + \sigma^*(*_i,y,z)
$$

\noindent where $\sigma^*(*_i,x,y) \equiv \sigma^*(\hat{*}_i)$ for the auto-op corresponding to the operation $y *_i z$.

It is clear that for example the following are CoSMs

\begin{itemize}
\item the length of the shortest program for computing a given argument / implementing a certain function (meaning $\sigma(y)$ is the length of the shortest self-delimiting program for computing $y$ on a given reference machine, and $\sigma^*(*,y,z)$ is the length of the shortest program for applying $y$ to $z$ on that reference machine)
\item the runtime of the fastest program for computing a given argument / implementing a certain function
\item the effort (e.g. space or time complexity) expected to be required for finding the shortest program for computing a given argument / implementing a given function
\end{itemize}

The decomposition involved in a CoSM has two aspects: producing an entity via combining simpler parts, and doing so in a manner that most efficiently utilizes subcombinations produced in this process.   To see this formally, given an expression $E$ in the free algebra of operators $*_i$ over $\mathcal{E}$, let 

$$
\sigma^!(E) = \sum_{x \in E} \sigma(x) +  \sum_{*_i \in E} \sigma^*(\hat{*}_i)
$$

\noindent (where the sums are over all specific instances of terms and operators in the expression $E$, i.e. the $\in$ is abused to mean syntactic inclusion).    Let $r(E)$ denote the result of evaluating $E$.   Then it's easy to see that

\begin{theorem}
Where $(\sigma, \sigma^*)$ is any CoSM, we have
$$
\sigma( r(E) ) \leq \sigma^!(E)
$$
\end{theorem}

\begin{proof}
The sum on the rhs is over all specific instances of terms and operators; but in many cases there will be some repetitions here.   A sequence of operations for producing $r(E)$ that achieves cost reduction via using some of the same terms and operators for multiple purposes, will make the inequality strict.  On the other hand if each term or operator instance is used exactly once in producing $r(E)$ with minimum total cost, then one has equality.
\end{proof}

\subsection{Compositional Simplicity on Probability Distribution Space}

We can also extend this formalism to the space $\mathcal{P}_\mathcal{E}$ of probability distributions over entities $\mathcal{E}$, in a natural way.   To extend the operators $*_i$ from entities to probability distributions $p,q$ over entities, we can associate $a$ with  the probability that $a$ is the produced by uniformly selecting a random element from the set of entities produced via choosing a random combination $x *_i y$, where $x$ and $y$ are drawn according to the distributions $p$ and $q$.   The identity of  $\mathcal{P}_\mathcal{E}$  is then the Dirac delta distribution concentrated on $e$.  In this way $\mathcal{P}_\mathcal{E}$ becomes a valid entity-space with no special formal difference from the base entity-space $\mathcal{E}$.  So everything we say below about simplicity on entity spaces $\mathcal{E}$ applies perfectly well to entity spaces of the form $\mathcal{P}_\mathcal{E}$  also.

 If a multiset $S$ of entities contains entities with frequencies corresponding to $p$, then we say that $S$ {\it realizes} $p$.   Any distribution $p$ that assigns only a finite set of probability values will have a finite realizing multiset.  Any $p$ that corresponds to a valid probability measure will be approximated arbitrarily closely by distributions taking on finite sets of values; the limit of the realizing multisets of these approximating distributions may be considered the realizing multiset of $p$.

Algebraic relations connecting simplicity measures with probability distribution algebra may be articulated, e.g.

\begin{theorem}
$\sigma( \frac{p + q}{2}) \leq \frac{\sigma(p) + \sigma(q)}{2}$
\end{theorem}
\begin{proof}
Consider first the case where both $p$ and $q$ have finite realizing multisets $S_p$ and $S_q$.  If the elements of $\frac{S_p + S_q}{2}$ are computed via computing the elements of $S_p$ and the elements of $S_q$ separately, then equality is obtained.  If there are some economies obtainable via leveraging commonalities between $S_p$ and $S_q$, then strict inequality may be obtained.  As the option is always there to treat $p$ and $q$ separately in the computation, the inequality is never in the opposite direction.  

The more general case follows via taking limits.
\end{proof}

\subsection{Relative Compositional Simplicity}

It can also be useful to have a notion of relative compositional simplicity.  

Toward this end it is convenient to relativize the simplicity measurement of combinational operators into  $\sigma^{*|}: \mathcal{O} \times \mathcal{A} \rightarrow [0,\infty)$.  The possibility that $  \sigma^{*|}(*_i,y,z|w)  \neq \sigma^*(*_i,y,z)$ is a essentially a way of introducing catalysis and other trans-binary emergent phenomena into the formalism.  If $w$ is a catalyst for the combination of $y$ with $z$ according to $*_i$ then one would have  $  \sigma^{*|}(*_i,y,z|w)  \leq \sigma^*(*_i,y,z)$. 

Where $(\sigma,\sigma^*)$ is any CoSM,

\begin{definition}
The {\it  relative compositional simplicity}  $\sigma(x|w) : \mathcal{E}^2 \rightarrow [0,\infty]$ is defined via

$$
\sigma(x| w) = min( min_{y,z: x = y *_i z }  h(y,z |w), min_{w,y: x = w*_i y} h(w,y) ,  min_{w,y: x = y *_i w } h(y,w)   )
$$

\noindent  where $h(y,z |w) = \sigma(y|w) + \sigma(z|w) + \sigma^{*|}(*_i,y,z|w)  $ and $h(y,z) = h(y,z|e)$.

\end{definition}

\noindent To understand this definition in a simple way, consider costs of the form $\sigma(a|b) = \sigma(a)$ and $ \sigma^{*|}(*_i,a,b|c))$ as ``elementary costs."   Then the relative compositional simplicity $\sigma(x| w)$ is the cost of the minimum-total-elementary-cost sequence of operations that starts with $w$ and produces $x$.

\begin{theorem}   
Two properties possessed by any relative compositional simplicity measure are given as follows:
\begin{enumerate}
\item  $\sigma(x) = \sigma(x|e)$.
\item $\sigma(y|z) \leq min_w ( \sigma(y|w) +\sigma(w|z) )$
\end{enumerate}
\end{theorem}
\begin{proof}
The first property is obvious.

The second property follows from the characterization of the relative compositional simplicity $\sigma(y|zz)$ as the cost of the minimum-total-elementary-cost sequence of operations that starts with $z$ and produces $y$.   In the case that this minimum-total-elementary-cost sequence is decomposable into a subsequence that uses $z$ and some entities $S_{zw}$ to produce $w$, and then uses $w$ and some other entities $S_{wy}$ to produce y (where $ S_{zw} \cap S_{wy}$ is empty)  then equality is obtained.   Otherwise the inequality is strict, as there is some lower-total-elementary-cost way to get from $z$ to $y$.
\end{proof}

The first property in this Theorem bears an obvious resemblance to the basic equations of dynamic programming.   One could use this equation to calculate $\sigma(y,z) $ using dynamic programming methods, but one would need to do this on a large graph with a node corresponding to all elements of $\mathcal{E}$ , unless one has prior knowledge to prune this down.

In the perspective given here, a simplicity measure on a combinational system is a sort of combination-quantification that cooperates nicely with the combinational operations of the system, so that intuitively speaking the simplicity of some object is the simplicity of the simplest combination that can produce it according to the operations of the system.   This general concept gives rise to some simple mathematical properties which are obeyed by standard measures regarding program length and runtime, but are also obeyed by a host of other measures as well.  

What I suggest, from a philosophical and practical view, is that it is the general properties of CoSMs outlined here  that are most important for Occam's Razor, rather than the specific simplicity measure chosen ( that obeys these properties).

\section{Multisimplicity}

Given that there are so many different reasonable ways to measure simplicity, it doesn't make sense to assume a given intelligent system or analysis process should rely on just one of them.  Rather, we need a multisimplicity framework, encompassing situations in which multiple, perhaps interrelated simplicity measures are used together to assess situations.

Formally: What if one has, instead of a single simplicity measure $(\sigma,\sigma^*)$, a family of measures $(\sigma_j,\sigma_j^*)$ that one cares about?  Given a list $*_i$ of combinatory operators, each $(\sigma_j,\sigma_j^*)$ may be defined according to a certain operation-set ${\mathcal O}_j = (*_{j1}, *_{j2}, \ldots, *_{jn_j})$, where the various  ${\mathcal O}_j$ may be identical, disjoint or overlapping.

It turns out that in order to explore this sort of multidimensional simplicity landscape, it is interesting to generalize the one-dimensional simplicity framework a bit, and consider multivalued quantitative measures or {\it multisimplicity measures} $\mu: \mathcal{E} \rightarrow [0,\infty)^\infty$ and $\mu^*: \mathcal{O} \rightarrow [0,\infty)^\infty$.   A multisimplicity measure may potentially assign an entity or auto-op a set of multiple values.   Given a list $*_i$ of combinatory operators, we can then look at a set of multisimplicity measures $(\mu_j,\mu_j^*)$ may be defined according to a  certain operation-set ${\mathcal O}_j = (*_{j1}, *_{j2}, \ldots, *_{jn_j})$, where the various  ${\mathcal O}_j$ may be identical, disjoint or overlapping.

The family  $(\vec{\mu}, \vec{\mu}^*) = ( (\mu_1,\mu_1^*), (\mu_2,\mu_2^*)  , \ldots )$ may be called a CoSMOS ("Compositional Simplicity Measure Operating Set") if

$$
\vec{\mu}(x)  \subseteq {\mathcal F}  \{  m_j(x) , j = 1 \ldots \}
$$

\noindent where

$$
m_j(x) = \{ \mu_j(y) + \mu_j(z) + \mu_j^*(*_i,y,z) | y *_i z = x \}
$$

\noindent where $\mu^*(*_i,x,y) \equiv \mu^*(\hat{*}_i)$ for the auto-op corresponding to the operation $y *_i.z$.  The notation here is: ${\mathcal F}$ denotes the Pareto frontier of the multiobjective minimization problem corresponding to its arguments, where each argument is a set $m_j(x)$ of numerical values, depending on $x$ and $j$.   Each $m_j(x)$ defines a minimization problem: namely, the problem of finding the smallest element in the set of values it identifies.  

In the context of a CoSMOS, the structures $\vec{\mu}(x) = \mu_1(x), \mu_2(x), \ldots $ and  $\vec{\mu}^*(*_i) = \mu_1^*(*_i), \mu_2^*(*_i), \ldots $ may be referred to as the {\it simplicity bundles} of $x$ and $*_i$ respectively.

One can then formulate a natural generalization of the usual Occam's Razor: {\bf Seek hypotheses whose simplicity bundles are Pareto optimal.}.  I.e., model your environment and experience using CoSMOS.

The theory of compositional simplicity can be relatively straightforwardly extended from CoSMs to CoSMOS.

E.g., given an expression $E$ in the free algebra of operators $*_i$ over $\mathcal{E}$, let $r(E)$ denote the result of evaluating $E$ as above. and let

$$
\mu_j^!(E) = \sum_{x \in E} \mu_j(x) +  \sum_{*_i \in E} \mu_j^*(\hat{*}_i)
$$

\noindent (where as above the sums are over all specific instances of terms and operators in the expression $E$, i.e. the $\in$ is abused to mean syntactic inclusion).     Let $\vec{\mu}^!(E) = \mu_1^!(E), \mu_2^!(E), \ldots  $
 
Let $\preccurlyeq$ denote reflexive Pareto dominance, extended to sets via interpreting $S \preccurlyeq T$ to mean that  $\forall s \in S, t \in T : s \preccurlyeq t$ .  Then it's easy to see that

\begin{theorem}
Where $(\vec{\mu}, \vec{\mu}^*)$ is any CoSMOS, we have
$$
 \vec{\mu}( r(E) ) \preccurlyeq    \vec{\mu}^!(E) 
$$
\end{theorem}
\begin{proof}
The simplicity bundle on the lhs, in each of its coordinates, sums up the cost relevant to simplicity measure $j$ over all specific instances of terms and operators; but in many cases there will be some repetitions here.   A sequence of operations for producing $r(E)$ that achieves cost reduction via using some of the same terms and operators for multiple purposes, will achieve lower cost in at least one dimension, thus producing the required dominance.  On the other hand if each term or operator instance is used exactly once in producing $r(E)$ with minimum total cost, then one has equality.
\end{proof}

\begin{theorem}
Where $(\vec{\mu}, \vec{\mu}^*)$ is any CoSMOS, extended to probability distributions in the natural way, $\vec{\mu}( \frac{p + q}{2}) \preccurlyeq  \frac{\vec{\mu}(p) + \vec{\mu}(q)}{2}$
\end{theorem}
\begin{proof}
Consider first the case where both $p$ and $q$ have finite realizing multisets $S_p$ and $S_q$.  If the elements of $\frac{S_p + S_q}{2}$ are computed via computing the elements of $S_p$ and the elements of $S_q$ separately, then equality is obtained.  If there are some economies obtainable via leveraging commonalities between $S_p$ and $S_q$, with respect to at least one dimension (one simplicity measure $j$), then strict dominance may be obtained.  As the option is always there to treat $p$ and $q$ separately in the computation, the inequality is never in the opposite direction.  

As in the one-dimensional case, the more general case follows via taking limits.
\end{proof}

Where $(\vec{\mu}, \vec{\mu}^*)$ is any CoSMOS

\begin{definition}
The {\it  relative compositional multisimplicity}  $\vec{\mu}(x|w)$ is defined via

$$
\vec{\mu}(x| w)  =  {\mathcal F}  \{  M_j(x) , j = 1 \ldots \}
$$

\noindent where

$$
M_j(x) = M^1_j(x) \cup M^2_j(x) \cup M^3_j(x)
$$

\noindent for

\begin{itemize}
\item $ M^1_j(x) = \{ \nu_j(y,z |w) | y *_j z = x \}$  
\item $ M^2_j(x) = \{ \nu_j(w,y) | w *_j y = x \}$  
\item $ M^3_j(x) = \{ \nu_j(y,w) | y *_j w = x \}$  
\end{itemize}

\noindent  where $\vec{\nu}_j(y,z ) =\vec{ \mu}(y|w) + \vec{\mu}(z|w) + \vec{\mu}^{*|}(*_i,y,z|w)  $

\end{definition}

\noindent In other words, the relative compositional simplicity $\vec{\mu}(x| w)$ is the set of Pareto-optimal simplicity bundles corresponding to processes that start with $w$ and produce $x$.  

The same dynamic-programming-esque property holds here as in the one-dimensional case, as is seen from a similar argument. 

\begin{theorem}   
Two properties possessed by any relative compositional multisimplicity measure are given as follows:
\begin{enumerate}
\item  $\vec{\mu}(x) = \vec{\mu}(x|e)$.
\item $\vec{\mu}(y|z)  \preccurlyeq   \vec{\mu}(y|w) + \vec{\mu}(w|z) $ for every $w$
\end{enumerate}
\end{theorem}

\section{Pattern From Simplicity}

Now we can build up the concept of {\it pattern} from the underlying framework of combination and simplicity developed above -- a critical step in the context of "patternist" approaches to cognition and complexity \cite{EGI1} \cite{GoertzelHP}.  This yields a notion of pattern that is compatible with previous approaches, but also goes broader and deeper in a way that makes clearer the open-endedness of the pattern concept and the richness of the dependence of pattern on perspective.

Let $(\vec{\sigma}, \vec{\sigma}^*) =(  (\sigma_1,\sigma_1^*), (\sigma_2,\sigma_2^*) ) $, where $(\sigma_1,\sigma_1^*)$ and $(\sigma_2,\sigma_2^*)$ are CoSMs with corresponding operator-sets $\mathcal{O}_1$, $\mathcal{O}_2$, with $\mathcal{O}_1 \subset \mathcal{O}_2$.   Denote $h_{1j}(y,z |w) = \sigma_1(y|w) + \sigma_1(z|w) + \sigma_j^{*|}(*_i,y,z|w)$ similarly to the definition of $h$ above (but noting that the first two terms use $\sigma_1$ and the third term $\sigma_j$).   

Given this setup, we may define {\it pattern} as follows: the pair $(y,z)$ is a ${\bf pattern}$ in $x$ relative to multisimplicity measure $(\vec{\sigma}, \vec{\sigma}^*)$ and context $w$ with intensity (fuzzy degree)

$$
I_{y,z}^{(\vec{\sigma}, \vec{\sigma}^*)}(x|w) = \frac{\sigma_1(x|w) - h_{12}(y,z|w) } {\sigma_1(x|w)}
$$ 

\noindent We can then say that $(y,z)$ is a pattern in $x$ (relative to $w$) if the degree $I_{y,z}^{(\vec{\sigma}, \vec{\sigma}^*)}(x|w) > 0$. 

The measure $(\sigma_1,\sigma_1^*)$ here plays the role of a "base simplicity measure", whereas $(\sigma_2,\sigma_2^*)$ is an additional simplicity measure that incorporates additional combinatory operations.   A pattern in $x$ is then a way of producing $x$ using these additional combinatory operations, that provides a simpler way (according to the additional simplicity measure) of getting $x$ than was possible using combinatory operations associated with the base simplicity measure. 

For instance, if $x$ is an image file, $(\sigma_1,\sigma_1^*)$ could be a program-length measure corresponding to an operator-set comprising operators that arrange pixels beside, above and below each others.    $(\sigma_2,\sigma_2^*)$ could be a program-length measure corresponding to a larger operator set that also includes a greater variety of programmatic operations such as conditionals, arithmetic operators and copying operations.   In that case, typical image compression programs would be representable via equations of the form $y *_j z = x$ where $*_j$ include the additional operators associated with $(\sigma_2,\sigma_2^*)$.  Setting $w=e$ for simplicity, in this case $\sigma_1(x)$ would correspond roughly to the number of pixels in the image (multiplied by the amount of memory required to store a single pixel); whereas $h_2(y,z)$ would correspond to a program for computing $x$ using the operators associated with $(\sigma_2,\sigma_2^*)$.   For this program to constitute a pattern it would need to compress $x$ to a shorter length than is possible using only the operators associated with $(\sigma_1,\sigma_1^*)$ .

This is conceptually similar to the definition of pattern previously given in \cite{GoertzelHP} and earlier publications, but digs a little deeper.   Whereas in prior work the simplicity measure was left as an arbitrary function, here we are giving a specific model for what is a simplicity measure, thus complicating things a bit but obtaining a finer-grained notion of pattern.

\subsection{Multipatterns}

One can also extend the pattern concept to multisimplicity measures in a natural way.   Let $(\vec{\mu}, \vec{\mu}^*) =(  (\mu_1,\mu_1^*), (\mu_2,\mu_2^*), \ldots ) $, where $(\mu_1,\mu_1^*)$ and $(\mu_j,\mu_j^*)$ for $j > 1$ are CoSMs with corresponding operator-sets $\mathcal{O}_1$, $\mathcal{O}_j$, with $\mathcal{O}_1 \subset \mathcal{O}_j$ for all $j>1$.   We then define the {\it pattern vector}

$$
\vec{I}_{y,z}^{(\vec{\mu}, \vec{\mu}^*)}(x|w)_j = \frac{\mu_1(x|w) - h_{1j}(y,z|w) } {\mu_j(x|w)}
$$ 

\noindent where  $h_{1j}(y,z |w) = \mu_1(y|w) + \mu_1(z|w) + \mu_j^{*|}(*_i,y,z|w)$  and

\begin{itemize}
 \item $(y,z)$ is a {\it full multipattern } in $x$ (relative to $\vec{\mu}$ and $w$) if $I_{y,z}^{(\vec{\mu}, \vec{\mu}^*)}(x|w)_j > 0$ for all $j >1$
 \item $(y,z)$ is a {\it mixed multipattern }in $x$ (relative to $\vec{\mu}$ and $w$) if $I_{y,z}^{(\vec{\mu}, \vec{\mu}^*)}(x|w)_j > 0$ for some but not all $j >1$
\end{itemize}

\noindent The intensity of a full multipattern $(y,z)$ in $x$ may be defined as the geometric mean of the $\vec{I}_{y,z}^{(\vec{\mu}, \vec{\mu}^*)}(x|w)_j$ for all $j$.   

One can also define the {\bf multipattern frontier} of $x$ as

$$
\vec{F}^{(\vec{\mu}, \vec{\mu}^*)}(x|w) = \mathcal{F} \{ P_j(x|w), j = 1, \ldots \} 
$$ 

\noindent where

$$
P_j(x|w) = \{ \frac{\mu_1(x|w) - h_{1j}(y,z|w) } {\mu_j(x|w)}, |  y *_i z = x \textrm{ for some } i \}
$$

\noindent Intuitively, an element of the multipattern frontier is a pair $(y,z)$ whose degree of pattern-intensity according to any one of the given simplicity measures, cannot be increased (by modifying $y$ or $z$ to some nearby entity) without decreasing its pattern-intensity according to some other of the simplicity members.

In pleasant situations, the multipattern frontier will consist of full or mixed multipatterns; but it seems this is not guaranteed to always be the case.

\section{Hierarchy from Pattern}

 Among the many kinds of "metapatterns" via which patterns may be organized, the commonality and centrality of hierarchical structure in physical and biological systems has been frequently observed  and has been cited as part of the reason for the effectiveness of deep neural architectures for various data analytics and reinforcement learning problems \cite{Lin_2017}.

With this in. mind, it's interesting to observe that the concept of pattern can be used to build up natural hierarchical structure on the space $\mathcal{E}$, deriving directly from the notion of simplicity underlying pattern.

\subsection{Hierarchy via Pattern Transitivity}

A very straightforward way to get hierarchy from pattern is via transitivity such as

\begin{theorem}
Suppose
\begin{itemize}
\item $(x,y)$ is a pattern in $\{a,b\}$
\item $(a,b)$ is a pattern in $z$
\end{itemize}
Then $(x,y)$ is a pattern in $z$
\end{theorem}

\noindent This sort of hierarchy certainly exists in cognitive systems and other complex systems, but is only one of the many routes from pattern to hierarchical structure.

\subsection{Subpattern Hierarchies}

A subtler sort of hierarchy is obtained if we use the notion of pattern to define a binary operation $\leq$ on $\mathcal{E}$, the {\it subpattern relation} defined relative to $(\vec{\sigma}, \vec{\sigma}^*)$ and $w$ via

$$
x \leq y \iff max_z I_{x,z}^{(\vec{\sigma}, \vec{\sigma}^*)}(y|w) > 0
$$

\noindent If $x \leq y$, we will say that $x$ is a {\bf compositional subpattern} of $y$.  I.e., this means $x$ can be combined with some other entity $z$ to form a pattern in $y$.   

Similarly we can define a {\it submultipattern relation} via

$$
x \leq y \iff max_z min_j I_{x,z}^{(\vec{\sigma}, \vec{\sigma}^*)}(y|w)_j > 0
$$

\noindent -- this means that $x$ can be combined with some $z$ to form a full multipattern in $y$, i.e. $x$ is a {\bf compositional submultipattern} of $y$.

The notion of a {\it subpattern hierarchy} is then formally reflected by the assertion that, under reasonable conditions, the  subpattern and submultipattern relations are {\it near partial orders}, so that e.g. if $x \leq y$ and $y \leq z$ are both true then $x \leq z $ is almost true; and so that if $x \neq y$ and $x \leq y$ then it's not possible for $y \leq x$.  

We will  show here that, if the operators $*_i$ in the combinational computational model described above have the right associativity and cost-associativity properties, then the partial-order property "almost holds" for compositional subpatterning.   Alternatively, if one has nonassociative or non-cost-associative operators, one can still get the partial order property to almost hold if one assumes there are some combinational operators that embody abstract combinatory-logic operators.    What we mean by "almost hold" is as follows,

\begin{definition}
We will say that the subpattern relation $\leq$ is an {\bf approximate partial order} on ${\mathcal E}$ if it is reflexive and antisymmetric, and there is some constant $c>0$ so that

$$
x \leq y ,  y \leq z \rightarrow max_w I_{x,w}(z) \geq -c
$$

\end{definition}

\noindent An analogous definition obtains for submultipatterns.

A collection of patterns on which the subpattern relation is an approximate partial order may be thought of as a "subpattern hierarchy" (where of course the strength of the hierarchical-ness is inversely proportional to the constant $c$).

 We suspect that, rather than just being complex and inelegant mathematical concoctions, our observations about routes to subpattern hierarchy are actually indicative of dynamics via which hierarchical structure emerges in different subnetworks of complex cognitive networks.    We conjecture that these two different routes to subpattern hierarchy may have specific relevance to human-like cognitive systems, in the sense that the operations involved in e.g. perception and action are evidently associative and cost-associative, whereas some of the operations involved in abstract cognition may not be; however, abstract cognition is more closely associated with abstract combinatory-logic style manipulations.  Thus it seems plausible that in human-like minds, pattern hierarchy is achieved via the associative route in perception and action oriented subsystems (and perhaps others), whereas by the combinatory-abstraction route in certain more abstractly-focused cognitive subsystems.

\subsubsection{Subpattern Hierarchy via Cost-Associativity}

If the operator-set $*_i$ is mutually associative, then we will say that

\begin{definition}
The mutually associative operator-set $\{*_i\}$ is {\bf approximately cost-associative}  relative to $\sigma$ if there is some constant $c>0$ so that

$$
| C_1(x,y,z) - C_2(x,y,z)    | < c
$$

\noindent where

\begin{itemize}
\item $C_1(x,y,z) = min_{i,j} ( \sigma^*(*_i, y, z) + \sigma^*(*_j, x, y *_i z))$ 
\item $C_2(x,y,z) = min_{i,j} ( \sigma^*(*_i, x*y, z) + \sigma^*(*_j, x, y)  )$
\end{itemize}

\end{definition}

Regarding the route from associativity to partial ordering, the potentials are as follows:

\begin{theorem} \label{thm:2}
The subpattern relation is an approximate partial order (with bound $c$) on $\mathcal{E}$ if:  The operations $*_i$ are approximately cost-associative (with bound $c$).  
\end{theorem}

\begin{proof}

First, it's straightforward to show that $y \leq x$ and $x \leq y$ can hold simultaneously iff $x=y$.  

To have $y \leq x$, one must have 

$$
\sigma_1(y) + \sigma_1(z) + \sigma_2^*(*_i,y,z) \leq \sigma_1(x)
$$ 

\noindent for some $z$ so that $x=y *_j z$ (for some allowable $j$); but, to have $x \leq y$, one must have 

$$
\sigma_1(x) + \sigma_1(w) + \sigma_2^*(*_i,x,w) \leq \sigma_1(y)
$$ 

\noindent for some $w$ so that $y=x *_k w$ (for some allowable $k$).  So, we must have 

$$
0 \leq \sigma_1(w) + \sigma_2^*(*_i,y,z) \leq \sigma_1(y) - \sigma_1(x) \leq - \sigma_1(z) - \sigma_2^*(*_i,x,w) \leq 0
$$

\noindent , meaning $\sigma_1(w) = \sigma_1(z) = 0$ so $w=z=e$, so $x=y$.

Next, we show "approximate transitivity," i.e. that $x \leq y$ and $y \leq z$ implies $max_u I_{x,u}(z) \geq -c$.

Firstly, $x \leq y$ means that for some $w$, we have $x *_i w = y$ and 

\begin{equation}  \label{eqn:1}
\sigma_1(x) + \sigma_1(w) + \sigma_2^*(*_i,x,w) \leq \sigma_1(y)
\end{equation}

Similarly, $y \leq z$ means that for some $v$, we have $y *_j v = z$ and 

\begin{equation} \label{eqn:2}
\sigma_1(y) + \sigma_1(v) + \sigma_2^*(*_j,y,v) \leq \sigma_1(z)
\end{equation} 

How then can we show $max_u I_{x,u}(z) \geq -c$?  We can do this via demonstrating some $u$ so that $x *_k u = z$ (for some allowable $k$) and

\begin{equation} \label{eqn:3}
\sigma_1(x) + \sigma_1(u) + \sigma_2^*(*_k,x,u) \leq \sigma_1(z) + c
\end{equation}

\noindent We claim that it works to set $k=i$ and $u = w *_j v$.  

By associativity we have

$$
z = (x *_i w) *_j v = x *_i (w *_j v)
$$

By cost associativity, we have

$$
|  (\sigma_1(y) + \sigma_1(v) + \sigma_2^*(*_j,y,v)  ) - (\sigma_1(x) + \sigma_1(u) + \sigma_2(*_i,x,u)) | < c
$$

\noindent meaning that to show Equation \ref{eqn:3} it suffices to show

\begin{equation} \label{eqn:4}
\sigma_1(y) + \sigma_1(v) + \sigma_2^*(*_j,y,v) \leq \sigma_1(z) + c_1
\end{equation}

\noindent holds for $c_1 = 2c$, which is clear since according to Equation \ref{eqn:2} this holds for $c_1=0$.

\end{proof}

To extend this idea to multipatterns we can say

\begin{definition}
The mutually associative operator-set $\{*_i\}$ is {\bf approximately cost-associative} relative to $\vec{\mu}$ and metric $d$ if there is some constant $c>0$ so that

$$
d( C_1(x,y,z) - C_2(x,y,z)   ) < c
$$

\noindent where

\begin{itemize}
\item $C_1(x,y,z) = \mathcal{F} (  \vec{\mu}^*(*_i, y, z) + \vec{\mu}^*(*_j, x, y *_i z))$ 
\item $C_2(x,y,z) = \mathcal{F}(   \vec{\mu}^*(*_i, x*y, z) + \vec{\mu}^*(*_j, x, y)  )$
\end{itemize}

\end{definition}

\noindent Given this extended definition, then Theorem \ref{thm:2} can be extended to submultipattern relations, i.e.

\begin{theorem} \label{thm:2}
The submultipattern relation is an approximate partial order (with bound $c$) on $\mathcal{E}$ if:  The operations $*_i$ are approximately cost-associative (with bound $c$).  
\end{theorem}

\subsubsection{Subpattern Hierarchy via Combinatorial Abstraction}

Where cost-associativity does not hold, there is an alternate route to subpattern hierarchy via abstract combinatory-logic to partial ordering.  

The key step here is to introduce a combinator $\Gamma$ so that, for the combinational operators $*_{\alpha,i}, *_{\beta,j}$, we have

$$
x *_i ( (\Gamma *_{\alpha,i} w) *_{\beta,j} v ) = (x *_i w) *_j v 
$$

This is formally a bit obscure-looking, but the way to think about it conceptually is: When one needs to go from 

$$
\textrm{Upper} * (\textrm{Middle} * \textrm{Lower})
$$ 

\noindent to  

$$
( \textrm{Upper} * \textrm{Middle} ) * \textrm{Lower}
$$

\noindent, this construction lets us do it by inserting $\Gamma$ by the $\textrm{Middle}$, so that  

$$
\textrm{Upper} * ( (\Gamma * \textrm{Middle}) * \textrm{Lower}) =  ( \textrm{Upper} * \textrm{Middle} ) * \textrm{Lower}
$$

\noindent If this kind of "lifting" construct $\Gamma$ exists relative to the given combinational operators and has reasonably low cost, then one can use it to approximatively emulate the kind of associativity needed to get hierarchy. 

\begin{theorem} \label{thm:3}
If the list of $*$ operators includes operators $*_{\alpha,i}, *_{\beta,j}$ enacting the $\Gamma$ combinatory operation, and 

$$
\sigma(x) + \sigma( (\Gamma *_{\alpha,i} w) *_{\beta,j} v ) +  \sigma^*( *_i, x,  (\Gamma *_{\alpha,i} w) *_{\beta,j} v) )
 \leq 
 \sigma(w) + \sigma(v) + \sigma(*_i,x,w) + \sigma^*I (*_j, y, v) + c
$$

\noindent for some constant $c$, then the subpattern relation is an approximate partial order, in the sense that it is reflexive and antisymmetric,and

$$
x \leq y ,  y \leq z \rightarrow max_w I_{x,w}(z) \geq -c
$$

\noindent where $c>0$ is a constant defined via $c = \sigma( \Gamma) + \sigma(\hat{*}_\gamma)$.

\end{theorem}

\begin{proof}

Reflexivity and antisymmetry are clear.  

The maximizing subpattern realizing $max_w$ is 

$$
w^\# =  ( (\Gamma *_{\alpha,i} w) *_{\beta,j} v ) 
$$

From calculations similar to those in the proof of Theorem \ref{thm:2} we know

$$
\sigma(x) + \sigma(w) + \sigma(v) + \sigma^*(*_i,x,w) + \sigma^*(*_j, y, v) \leq \sigma(z)
$$

\noindent so putting this together with the assumption of the theorem we see

$$
\sigma(x) + \sigma( w^\# ) +  \sigma^*( *_i, x, w^\#) \leq \sigma(z) -c
$$

\noindent which yields the theorem.

\end{proof}

It is straightforward to extend these results to show the compositional multipattern relation is also a partial ordering under similar conditions.  Basically, if the conditions of the theorems hold for all the simplicity measures $\mu_j$ in a multisimplicity measure, then the essential conclusions of the theorems hold for the compositional multipattern relation.

\subsubsection{Possible Practical Relevance of Routes to Subpattern Hierarchy}

Not all useful combinational operators need be associative or cost-associative, of course.   However, there is significant convenience in having patterns fall into a subpattern hierarchy, and in this sense there is significant convenience in having cost-associative combinational operations.  In a complex cognitive systems there are likely to be some important combinational operators that are associative and cost-associative and some that are otherwise, and this will determine which aspects of the cognitive system are effectively modeled in terms of subpattern hierarchies via this particular route.

The biochemistry and molecular biology inspired combinatory operations referenced above are all associative in nature, and the typical transformations involved in visual or auditory processing, or movement control in robotics, are also associative.   In what contexts all these transformations are also cost-associative is subtler.  

On the other hand, the function-application operation $f * x = f(x)$ is generally not associative.   Function composition is associative, so if one can deal with a certain domain via various permuted compositions of a certain set of basic functions, one can operate within the domain of associative algebra.  But for cognitive processing this is not always the case.  Turing complete computation can be achieved using solely associative operators, which is interesting; but some things that can be done very concisely non-associatively become quite lengthy and elaborate if one must use purely associative operations.  The full cognitive-science and AI implications of this direction of thought are yet to be unraveled but we suspect they are significant in the domain of general-purpose abstract cognition under realistic resource constraints.

\section{Deriving Metric-Based Heterarchy from Subpattern Hierarchy}

Systems of patterns are generally characterized by both hierarchical and heterarchical structures (as argued extensively in \cite{Goertzel1993}).  We have modeled hierarchical structure above using a partial order; heterarchical structure, on the other hand, is effectively mathematically modelable using metric structure.  One important form of heterarchy is a network in which patterns are associated with other similar patterns, and similarity is merely a rescaling of distance.

Heterarchical structure can complement hierarchical structure when the latter is present, but can also exist when hierarchical structure is not there e.g. due to nonassociativity of the $*_i$.

One may also derive a metric structure on $\mathcal{E}$ from the subpattern hierarchy as given above, leveraging the underlying definition of simplicity.  This may be done in many ways; one approach is to associate two sets with each $x \in X$: the {\bf pattern intension} $int(x)$ consisting of all $y$ so that $x \leq y$ (according to the subpattern or submultipattern approximate-partial-order); and the  {\bf pattern extension} $ext(x)$ consisting of all $y$ so that $y \leq x$.   One  may then construct a metric associated with each of these; as a simple example, where $d_T$ is the Tanimoto distance \cite{Tanimoto1960}, one can set

\begin{itemize}
\item $d_I(x,y) = d_T(int(x), int(y))$
\item  $d_E(x,y) = d_T(ext(x), ext(y))$
\end{itemize}

\noindent One may then define a composite metric as

$$
d(x,y) = \alpha d_I(x,y) + (1-\alpha) d_E(x,y)
$$

\noindent This is one way to obtain a metric structure to the space of processes $\mathcal{E}$ in a way that coincides naturally with the hierarchical structure; we will look at complementary approaches involving Pareto frontiers below.   

Another alternate approach is to look more closely at the quantity used in the definition of the order $\leq$ given above, which we may name
 
$$
Q(y,x) = max_z I_{y,z}(x)
$$

\noindent This quantity measures the degree to which $y$ is a compositional subpattern in $x$.  If one normalizes the set of $Q(*,x)$ values over the pattern extension of $x$, then one obtains a probability distribution of $Q$ values over the pattern extension of $x$; and one can measure the distance between the pattern extension distribution associated with $x$ and the one associated with $y$ using the Hutchinson metric \cite{Hutchinson1981}.  One can do similarly for pattern intensions, and then create a composite metric analogous to the above.

The relation between this probabilistically defined (Hutchinson) metric and the previously specified crisply defined (Tanimoto) metric is unclear, though one expects they will often be strongly correlated.

In either case, it seems, one obtains a metric that plays nicely with the hierarchical structure defined by the partial order, which is important if one wishes to explore the general relation between hierarchy and heterarchy in systems of patterns.   The best way to work the formal details is not entirely clear and may depend on the context one is considering, but the conceptual route to defining pattern heterarchies that are consilient with pattern hierarchies clear -- and in the following section we will take this one step further.

\section{Lossy Multipattern Frontiers and Coherent Dual Networks}

Given a set $\vec{d} = (d_1,d_2, \ldots)$ of metrics on $\mathcal{E}$ (e.g. intensional and extensional metrics as two examples), we may define the {\it lossy multipattern frontier} of $x$ relative to $w$ as

$$
\vec{F}^{(\vec{\mu}, \vec{\mu}^*)}(x|w) = \mathcal{F}(  \{ P_j(x|w), j = 1, \ldots \}  \cup   \{ D^k_j(x|w), j = 1, \ldots \}   )
$$ 

\noindent where $P_j(x|w)$ is as above and

$$
D^k_j(x|w) = \{ \frac{k}{ k + d_j(y *_i z, x) } |  y , z : y *_i z \textrm{ is legal for some } i \}
$$

\noindent where $k>0 $ is a parameter ( the "personality parameter").  These are pairs $(y,z)$ whose outputs $w$ are near $x$ but not identical to $x$, w with the property that: They can't be made any more intensely patterns in $w$ according to any of the simplicity measures, and their output can't be nudged any closer to $x$ by any of the distance metrics, without screwing something up (i.e. making them less intense patterns in $w$ according to some simplicity measure, or moving their output further away from $x$ according to one of the distance metrics).

It is interesting to tweak the notion of intension so that $x$'s intension comprises the elements of $x$'s lossy multipattern frontier (multipattern-based intension), obtaining a definition of intensional distance that incorporates the above Pareto-optimum calculations.  Fuzzifying this idea, one can define the membership degree of $(y,z)$ in the intension of $x$ as its distance (according to some base metric $d$) from the closest element on the lossy multipattern frontier.   We may call the intensional distance derived this way the LMI (lossy multipattern intension) distance $d_{\textrm{LMI}}^{ \vec{\mu} , \vec{d}}$.  

The notion of intension is broad, and it can also be interesting to include "distance to the closest element on the lossy multipattern frontier" as one ingredient among multiple in determining the fuzzy membership degree of $(y,z)$ in $x$'s intension.   If the multipattern frontier is playing a significant role in determining the fuzzy degree, we can still consider this a sort of $d_{\textrm{LMI}}$.

One can use this to formalize the notion of alignment between hierarchical and heterarchical structures:

\begin{definition}
A {\bf fully coherent dual network} is a multisimplicity measure $\vec{\mu}$ together with a metric vector $\vec{d} = (d_I, d_2, \ldots)$ so that

\begin{equation} \label{eqn:dual}
d_{\textrm{LMI}}^{ \vec{\mu} , \vec{d}} = d_I  
\end{equation}

\noindent The {\bf degree of dual network coherence} of a pair $(\vec{\mu}, \vec{d})$ is

\begin{equation} \label{eqn:dual}
1 - \frac{ \int_{(x,y)\in \mathcal{E}^2} | d_{\textrm{LMI}}^{ \vec{\mu} , \vec{d}} (x,y) - d_I(x,y)|  dm(x,y)} { \int_{(x,y)\in \mathcal{E}^2} max( d_{\textrm{LMI}}^{ \vec{\mu} , \vec{d}} (x,y) , d_I(x,y) ) dm(x,y) }
\end{equation}

\noindent where $m$ is an assumed measure on $\mathcal{E}^2$.   (For a fully coherent dual network the degree is evidently 1).
\end{definition}

\noindent Where $d_{\textrm{LMI}}$ includes ingredients coming directly from the subpattern hierarchy (along with a component from the multipattern frontier), and/or where $d_2 = d_E$ so the subpattern hierarchy is involved via inclusion of extensional distance -- or when any of the $d_i$ are determined via pattern-based calculations in some other way -- then one has an interesting recursion so that being a fully coherent dual network requires that $d_I$ is a fixed-point of the operation $d_I \rightarrow d_{\textrm{LMI}}$ in a nontrivial way; and having a high but not full degree of dual network coherence requires $d_I$ has nearly fixed point behavior under this mapping.

Of course, the particulars of this sort of definition of a coherent dual network could be worked in various different ways.   Conceptually, however, what we have here is the first reasonably general, precise description of what it means for hierarchical and heterarchical pattern structures to be harmoniously aligned.   The hypothesis of a "dual network"  of this general nature underlying cognition was made in \cite{Goertzel1993} and elaborated in \cite{GoertzelHP} and has been fleshed out previously from an AI design, neuroscience and systems theory perspective, but has not been previously formalized even to the extent we have just done here (said bearing in mind that the current formalization is still quite high-level and is still near the beginning of what may be a long road regarding the connection of formal simplicity-theoretic notions with practical complex and intelligent systems).

\section{Conclusion}

Occam's Razor is a valuable heuristic, yet only really useful beyond the hand-waving level if paired with a deeper understanding of simplicity.  Grounding simplicity in physics or computation theory can be valuable for certain purposes, yet ultimately is conceptually unsatisfactory from a cognitive science perspective, because one doesn't want one's theory of cognition to depend on particular aspects of physics and computing infrastructure at such a basic foundational level.  (Instead, from a philosophical view, one would like to use a theory of cognition to explain how a mind constructs a world -- taking simplicity in some cases as a concept prior to physical or computing systems.)

An alternative approach is to place cognitive theory at the center,  and re-envision Occam's Razor as something like the following: {\it In order for a system of inter-combining elements to effectively understand the world, it should interpret itself and the world in the context of some array of simplicity  measures obeying certain basic criteria.  Doing so enables it to build up coordinated hierarchical and heterarchical pattern structures that help it interpret the world in a subjectively meaningful and useful way.}   In this paper we have put some meat on the bones of this idea, via presenting a formal axiomatic theory of simplicity, and showing that it relates to the computational approach to simplicity, and gives rise to hierarchical and heterarchical cognitive structures with appropriate properties.  This provides a more solid foundation for cognitive theories founded on notions of simplicity and pattern, such as those presented previously in \cite{GoertzelHP}.  

As formulated in the Introduction, we thus extend "when in doubt, prefer the simpler hypothesis" into the multisimplicity maxim {\bf when in doubt, prefer hypotheses whose simplicity bundles are Pareto optimal} , coupled with observation that {\bf doing so both permits and benefits from the construction of coherent dual networks comprising coordinated and consistent multipattern hierarchies and heterarchies}.

As an illustration of the conceptual value of this sort of approach, we  have used it here to explore the different ways in which pattern hierarchies can emerge in combinational systems: From combinational operations obeying the associative law, or from combinational operations including abstract combinatory-logic-style operations.   We have hypothesized that in human-like minds, the former is the route taken in the context of perception and action subsystems, whereas the latter is the route taken in the context of some more abstract cognitive subsystems.  We have also explored how these hierarchies can give rise to heterarchies which can then be coordinated with the hierarchies in precisely balanced ways.

Clearly there is much more to unravel along these lines, regarding the relationship between combination, simplicity, pattern, hierarchy, heterarchy and other foundational notions of cognitive structure and action.

\section*{Acknowledgements} 

Thanks are due to Alexey Potapov for carefully reading the initial version of this paper, thus prompting the radical overhaul of the formal sections resulting in the current version.  (Whatever blame or  credit  is due for the particulars of the new formal sections should not be directed to Alexey, however.)

\bibliography{bbm}
\bibliographystyle{alpha} 

\end{document}